\documentclass{IEEEtran4PSCC}
\usepackage{cite}
\usepackage{multicol}
\usepackage{multirow}
\usepackage[backref=page]{hyperref}
\hypersetup{
    colorlinks = true,
    citecolor=blue,
    linkcolor = blue,
}
\usepackage{amsthm,amsmath}
\usepackage{enumerate}
\usepackage{cuted}
\usepackage{mathtools}
\usepackage{graphics}

\usepackage{amssymb}
\graphicspath{{figures/}}
\usepackage[ruled]{algorithm2e}
\usepackage{color}
\usepackage{mathtools}

\newtheorem{theorem}{Theorem}
\newtheorem{corollary}{Corollary}

\newtheorem*{remark}{Remark}
\newtheorem*{note}{Note}
\usepackage{xpatch}
\usepackage{tikz}
\usepackage{tikz-network}
\makeatletter
\makeatletter
\let\old@ps@headings\ps@headings
\let\old@ps@IEEEtitlepagestyle\ps@IEEEtitlepagestyle
\def\psccfooter#1{%
 \def\ps@headings{%
 \old@ps@headings%
 \def\@oddfoot{\strut\hfill#1\hfill\strut}%
 \def\@evenfoot{\strut\hfill#1\hfill\strut}%
 }%
 \def\ps@IEEEtitlepagestyle{%
 \old@ps@IEEEtitlepagestyle%
 \def\@oddfoot{\strut\hfill#1\hfill\strut}%
 \def\@evenfoot{\strut\hfill#1\hfill\strut}%
 }%
 \ps@headings%
}
\makeatother

\begin{document}
\title{Data-Efficient Strategies for Probabilistic Voltage Envelopes under Network Contingencies}
%% To specify the authors when (number of affiliations <= 2)
\author{
\IEEEauthorblockN{Parikshit Pareek, Deepjyoti Deka and Sidhant Misra}}
% \IEEEauthorblockA{Theoretical Division,
% Los Alamos National Laboratory, NM, USA\\
% \{pareek,deepjyoti,sidhant\}@lanl.gov}
% }

\maketitle

\begin{abstract}
This work presents an efficient data-driven method to construct probabilistic voltage envelopes (PVE) using power flow learning in grids with network contingencies. First, a network-aware Gaussian process (GP) termed Vertex-Degree Kernel (VDK-GP), developed in prior work, is used to estimate voltage-power functions for a few network configurations. The paper introduces a novel multi-task vertex degree kernel (MT-VDK) that amalgamates the learned VDK-GPs to determine power flows for unseen networks, with a significant reduction in the computational complexity and hyperparameter requirements compared to alternate approaches. Simulations on the IEEE 30-Bus network demonstrate the retention and transfer of power flow knowledge in both N-1 and N-2 contingency scenarios. The MT-VDK-GP approach achieves over 50\% reduction in mean prediction error for novel N-1 contingency network configurations in low training data regimes (50-250 samples) over VDK-GP. Additionally, MT-VDK-GP outperforms a hyper-parameter based transfer learning approach in over 75\% of N-2 contingency network structures, even without historical N-2 outage data. The proposed method demonstrates the ability to achieve PVEs using sixteen times fewer power flow solutions compared to Monte-Carlo sampling-based methods. 
\end{abstract}


% Use this to place sponsorships

\thanksto{\noindent Authors are with Theoretical Division (T-5), Los Alamos National Laboratory, NM, USA.\{pareek,deepjyoti,sidhant\}@lanl.gov \\
The authors acknowledge the funding provided by LANL’s Directed Research and Development (LDRD) project: “High-Performance Artificial Intelligence" (20230771DI) and and The Department of Energy (DOE), USA under the Advanced grid Modeling (AGM) program. The research work conducted
at Los Alamos National Laboratory is done under the auspices of the National Nuclear Security Administration of the U.S. Department of Energy under Contract No. 89233218CNA000001}
 
\section{Introduction}
The steady-state of a power grid is determined by solving a set of nonlinear equations known as Alternating Current Power Flow (ACPF), which relates nodal voltages (states) to input nodal power injections. These equations lack closed-form expressions for nodal voltages and are typically solved iteratively using methods like the Newton-Raphson load flow (NRLF) \cite{molzahnsurvey}. Dealing with injection uncertainties complicates the problem, often leading to the use of the various model-based \cite{molzahnsurvey} and data-based \cite{jia2023tutorial,misra2018optimal,deka2017structure} approximations to solve the ACPF. Linear models, while interpretable, struggle to capture the full non-linearity of power flow, and lifted approaches tend to sacrifice interpretability and require extensive samples \cite{9537673}. 

Also, to assess the impact of network contingencies on system operation, probabilistic power flow (PPF) problems are solved with different topologies to aid in security constrained decision-making \cite{line2013}. Analytical methods such as Cumulants are used to assess random branch outage effect on node voltage magnitude \cite{singh2023}. Yet, the approach relies on assumptions on load uncertainty distributions and does not provide any guarantee on maximum/minimum voltage magnitude values. To overcome that, Monte-Carlo simulation (MCS)-based methods are proposed to work with arbitrary input uncertainties, which can provide statistical estimation guarantees as well \cite{zou2014}. The biggest limitation of these MCS based methods is the computational burden. Further, a significant limitation of data-driven machine learning models lies in their lack of topological awareness, resulting in limited adaptability to out of sample data \cite{9537673}, resulting from sparse topology changes from the base case. Practical power grids often experience contingencies or undergo regular maintenance, leading to occasional changes in topology \cite{chen2022review}. The number of potential power network configurations is combinatorial in nature, presenting substantial computational challenges in learning, particularly when coupled with load uncertainties \cite{Go}. 

To tackle this challenge and address the variability in topology, topology-aware data-driven methods have emerged as a promising solution. These methods either aim to learn a single model suitable for various network topologies \cite{Go,G1,G2} or acquire an average model that can be efficiently adapted to different topologies \cite{chen2022meta}. The incorporation of topology into these methods is primarily achieved by encoding it as input features \cite{G1,G2} or embedding it into a neural network (NN) structure \cite{owerko2020optimal}. In previous approaches focusing on topology encoding, features such as voltage differences between different topologies \cite{G2} and diagonal elements of the susceptance matrix \cite{G3} have been employed. While these methods can handle multiple topologies, their adaptability to untrained topologies is limited due to potential shortcomings in fully capturing topology properties during input vector construction. Another line of research integrates power grid topology within the forward-propagation process of Graph Convolutional Networks (GCNs), inherently equipped with topology features. These studies utilize constant graph convolution kernels derived from the Laplacian matrix \cite{G4} and the adjacency matrix of transmission lines \cite{donon2019graph,owerko2020optimal,Wang2020} to construct GCNs and approximate power flow mappings. 

In \cite{chen2022meta}, a meta-learning approach is presented, where an average model for optimal power flow over multiple topologies is learned and can be subsequently updated to achieve desired performance with new topology data. Despite enhancing GCN adaptability, these methods may not fully consider the physical relationships between nodes, limiting their applicability across various topological scenarios. Recently, \cite{gao2023physics} provided a critical overview of these methods and introduced a physics embedding approach using neighborhood aggregation in GCN. However, this GCN architecture requires a substantial amount of training data and has shown limited robustness against uncertain unseen topologies, along with acceptable error margins. Additionally, approaches aiming to create a single model for different network topologies implicitly assume that voltage is a continuous function of the network graph for a given load, that may not hold. 

In this paper, we attempt to develop a framework for obtaining probabilistic voltage envelopes (PVEs) via  efficient learning of voltage-power functions for power grid with topology changes. A PVE delineates the plausible range of a node voltage magnitude, with plausibility quantified in terms of probability. Fundamentally, it indicates the likelihood that the node voltage magnitude will fall within the PVE, given a specific probability, when the load originates from a predefined uncertainty set. This work focuses on reducing the training data requirement for learning power flow over a new topology by leveraging existing models learned for prior topologies. This is in line with operational realities where limited data is available for a newly realized network topology due to contingencies, while only a few prior network structures can be deeply analyzed. We build a set of transfer learning (TL) methods based on the recently proposed vertex degree kernel (VDK) \cite{pareek2023graph} and its ability to incorporate network structures into GP learning. A multi-task VDK (MT-VDK) is designed by combining VDKs from different existing models with a parameterized VDK for the new topology. Existing or learned VDKs have fixed hyperparameters, and a weight parameter is used over the sum of these VDKs. This means that the number of hyperparameters does not grow with the number of existing topologies considered. Furthermore, by separating the source and new model training, the proposed MT-VDK-GP has lower learning complexity than MT-GP \cite{tighineanu2022transfer}. Thus, the proposed method allows us to obtain PVEs with very few samples compared to ACPF sample-based envelop construction methods, while maintaining similar probabilistic bounds.

Main contributions of this paper can be summarized as: 

\begin{itemize}
 \item Developing a computationally efficient method to obtain probabilistic voltage envelope (PVE) for different network topologies and input distributions. The MT-VDK-GP is used for numerical evaluation of voltage envelopes with probabilistic guarantees. 
 \item Designing a multi-task vertex degree kernel (MT-VDK) to utilize available data/models in learning voltage-power function, in data-efficient manner, for new network structures. A simple hyperparameter transfer learning (HTL) mechanism has also been designed exploiting network structure-based VDK \cite{pareek2023graph}.
 \item Benchmarking the proposed transfer learning mechanisms for N-1 and N-2 contingencies with $\pm10\%$ uncertainty in all node injections, for IEEE 30-Bus system. In low data regime, proposed methods show more than two fold reduction in errors compared to cold-start VDK-GPs. 
\end{itemize}

In the subsequent sections of the paper, we delve into specific methodologies and strategies employed in our approach. The following section focuses on topology-aware power flow learning, where we describe how information about network topology is embedded in power flow learning. Section III explores transfer learning strategies applied to reduce the training data requirement, highlighting techniques for leveraging pre-trained models and knowledge transfer under network contingencies. In Section IV, we delve into probabilistic voltage envelopes (PVE), discussing their incorporation for uncertainty quantification in power flow predictions. We also prove the worst-case bound on PVEs obtained using the proposed GP-based approach. Following this, Section V presents our results, detailing the performance of the proposed methodologies. Finally, in Section VI, we conclude by summarizing our findings and discussing potential avenues for future research.

\section{Topology Aware Power Flow Learning}
This section begins with an overview of power flow in functional form, from GP learning perspective. Subsequently, we explore the intricacies of the topology-aware kernel design. The alternating current power flow (ACPF) model is used in polar coordinate form with node voltage referred to as $v_j=V_j\angle{\theta_j}$ ($V_j$ is magnitude, and $\theta_j$ is the phase angle). Hereafter, voltage is used to refer voltage magnitude. Injection refers to generation minus demand, while load is equal to negative injection. The injection vector is denoted as $s_j = [p_j;q_j]$. Here, $p_j$ ($q_j$) represents real (reactive) load/injection at the $j$-th node. The set of all nodes is denoted as $\mathcal{N}$ with $|\mathcal{N}|$ being the number of nodes. The set of branches for a given network structure or topology is denoted by $\mathcal{E}$. Total number of branches are denoted by $N$ with $N-1$ for single and $N-2$ for double branch outage contingencies. Standard ACPF equation for a given network structure represented by admittance matrix ($\rm Y$-Bus). Our goal is to represent ACPF as a map between a node voltage magnitude and injection vector
\cite{pareek2021framework,pareek2023graph}, as
\begin{align}\label{eq:v_gp}
 V(\mathbf{s}) = f(\mathbf{s}) + \varepsilon
\end{align}
where, $V(\mathbf{s})$ is node voltage magnitude measurement or NRLF solution at any target node of the network (to enhance brevity, the subscript '$j$' is omitted hereafter). In \eqref{eq:v_gp}, $f(\mathbf{s})$ is unknown and $\varepsilon$ is i.i.d Gaussian noise $\varepsilon \sim \mathcal{N}(0,\sigma^2_n)$. Note that ACPF is solved for a given Y-Bus i.e. for a given network topology, and thus \eqref{eq:v_gp} is valid for that network structure only. Now, we utilize GP to model a node voltage function, denoted as $f(\mathbf{s})$ and defined as a zero-mean Gaussian process, \cite{williams2006gaussian}:
\begin{align}\label{eq:gp}
 f(\mathbf{s}) \sim \mathcal{GP}\Big ( \mathbf{0}, K(S,S) \Big )
\end{align}
here, $K(\cdot,\cdot)$ represents the covariance, or kernel matrix, computed over the training samples, $K_{i,j}= k(\mathbf{s}^i,\mathbf{s}^j)$. The design matrix $S = [\mathbf{s}^1 \dots \mathbf{s}^i \dots \mathbf{s}^{N_s}]$ comprises $N_s$ i.i.d injection vector samples forming the training set $\{\mathbf{s}^i,V^i\}^{N_s}_{i=1}$. See \cite{williams2006gaussian,pareek2021framework,pareek2023graph} for more details. This kernel function operates on injection vectors and incorporates hyperparameter vector $\boldsymbol{\theta}$. These hyperparameters can be estimated (optimally) using maximum log-likelihood estimation (MLE) for exact inference \cite{williams2006gaussian}. Once learned, the GP provides predictions of the function's mean and variance, represented as:
\begin{align}
 \mathbb{E}[f(\mathbf{s})] & = \mu_f(\mathbf{s}) = \mathbf{k}^T \boldsymbol{\alpha} \label{eq:vmean} \\
 \mathbb{V}[f(\mathbf{s})] & = \sigma^2_f(\mathbf{s}) = k(\mathbf{s},\mathbf{s})-\mathbf{k}^T\rm A \mathbf{k} \label{eq:vsigma}
\end{align}
with $\rm A = [K(S,S)+\sigma^2_n I]^{-1}$ and $\boldsymbol{\alpha} =\rm{A} \mathbf{V}$ being constant after learning for a given hyperparameter vector $\boldsymbol \theta$. The variance \eqref{eq:vsigma} is a direct indicator of quality of learning by quantifying the confidence in predicted voltage for an injection. A GP model, provides a predicted point, denoted as $\mu_f(\mathbf{s})$ in \eqref{eq:vmean}, and also offers a measure of confidence through the predictive variance, $\sigma^2_f(\mathbf{s})$. Put simply, this confidence can be interpreted as indicating a 95\% likelihood that the actual voltage solution obtained by solving the power flow at a specific injection point will fall within the range of $\mu_f(\mathbf{s}) \pm 2\sigma_f(\mathbf{s})$. Therefore, a lower value of $\sigma^2_f(\mathbf{s})$ corresponds to a more reliable GP model.\cite{williams2006gaussian}.

Referring to \eqref{eq:gp}, it is clear that the learning process is directly influenced by the choice of the kernel function.
If the chosen Kernel is topology-agnostic and valid only for the topology of the training dataset, any network change or contingency would require relearning the entire model from scratch and require sufficient data from the new topology. The first step towards working with multiple network structures involves making the GP learning model aware of network topology. The following subsection introduces the recently developed vertex degree kernel (VDK) \cite{pareek2023graph} and elucidates methods for adapting VDK to various network structures.

\begin{figure}[t]
\centering
\resizebox{\columnwidth}{2.5cm}{
 \begin{tikzpicture}
 \Vertex[IdAsLabel,RGB,color={252,80,30},size = 0.5,x = -0.8]{1} 
 \Vertex[IdAsLabel,RGB,color={120,239,30},x=-0.8,y=1,size = 0.5]{3}
 \Vertex[IdAsLabel,RGB,color={120,239,30},x=-0.8,y=-1,size = 0.5]{2}
 \Vertex[IdAsLabel,x=0.2,RGB,color={120,239,246},size = 0.5]{12}
 \Vertex[IdAsLabel,x=0.8,y=1,RGB,color={120,239,246},size = 0.5]{5}
 \Vertex[x=0.2,y=-1.5,size = 0.5,Pseudo]{7}
 \Vertex[x=0.8,y=-1,size = 0.2,Pseudo]{13}
 \Text[x=-2.5,y=0]{\footnotesize $k_1([\textcolor{blue}{\mathbf{s}_1};\mathbf{s}_2;\textcolor{purple}{\mathbf{s}_3}],\cdot)$}
 \Text[x=-2.5,y=-1]{\footnotesize $k_2([\mathbf{s}_1;\textcolor{blue}{\mathbf{s}_2};\mathbf{s}_{12}],\cdot)$}
 \Text[x=-2.7,y=1]{\footnotesize $k_3([\textcolor{purple}{\mathbf{s}_1};\textcolor{blue}{\mathbf{s}_3};\mathbf{s}_{5};\mathbf{s}_{12}],\cdot)$}
 % Edges 
 \Edge(1)(2)\Edge(1)(3)\Edge(2)(12)
 \Edge(3)(12) \Edge(3)(5)
 \Edge[style={dashed}](12)(7) \Edge[style={dashed}](5)(13)
 % --------- Second Topology --------------
 \Vertex[RGB,color={252,80,30},size = 0.5,x=4.5,label=1]{1a} 
 \Vertex[RGB,color={120,239,30},x=4.5,y=1,size = 0.5,label=3]{3a}
 \Vertex[label=2,RGB,color={120,239,30},x=4.5,y=-1,size = 0.5]{2a}
 \Vertex[label=12,x=5.4,RGB,color={120,239,246},size = 0.5]{12a}
 \Vertex[label=5,x=6.3,y=1,RGB,color={120,239,246},size = 0.5]{5a}
 \Vertex[x=5.4,y=-1.5,size = 0.5,Pseudo]{7a}
 \Vertex[x=6.3,y=-1,size = 0.2,Pseudo]{13a}
 % Edges 
 \Edge(1a)(2a) 
 %\Edge(1a)(3a)
 \Edge(2a)(12a)
 \Edge(3a)(12a) \Edge(3a)(5a)
 \Edge[style={dashed}](12a)(7a) \Edge[style={dashed}](5a)(13a)
 % Text
 \Text[x=3.3,y=0]{\footnotesize $k_1([\textcolor{blue}{\mathbf{s}_1};\mathbf{s}_2],\cdot)$}
 \Text[x=3.1,y=-1]{\footnotesize $k_2([\mathbf{s}_1;\textcolor{blue}{\mathbf{s}_2};\mathbf{s}_{12}],\cdot)$}
 \Text[x=3.1,y=1]{\footnotesize $k_3([\textcolor{blue}{\mathbf{s}_3};\mathbf{s}_{5};\mathbf{s}_{12}],\cdot)$}
\end{tikzpicture}}
\vspace{-1.5em}
 \caption{Different sub-kernels of VDK for different topologies of a part of 118-Bus system. Right hand side network shows the changes in kernels due to line 1-3 outage under N-1 contingency. Purple colored injections on the left are absent from sub-kernels on right hand side network.}
 \vspace{-1.5em}
 \label{fig:VDK}
\end{figure}
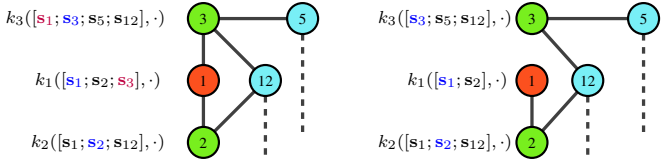

\subsection{Network Structured Kernel Design}
Recent work \cite{pareek2023graph} exploits the network graph to design an additive kernel by utilizing nodal sub-kernels that focus on neighborhood injections. Intuitively, the VDK assumes that injections from nodes which are far away from each other do not have a correlated effect on any node voltage in the network \cite{pareek2023graph}. 
For a given node $n$, neighborhood injections come from the nodes directly connected to $n$. The set of variables for neighborhood injections is defined as $\{\mathbf{x}_n = \{\mathbf{s}_m | (n, m) \in \mathcal{E}\}\}$. A VDK is
\begin{align}\label{eq:VDK}
 k_{\mathcal{E}}(\mathbf{s}^i,\mathbf{s}^j) = \sum^{|\mathcal{N}|}_{n =1} k_n(\mathbf{x}^i_n,\mathbf{x}^j_n) 
\end{align}
where, $k_n(\mathbf{x}^i_n,\mathbf{x}^j_n)$ is a neighborhood kernel working over $\mathbf{x}_n \subset \mathbf{s}$. Using the power flow function's smoothness property, we select \textit{square exponential kernel} with hyperparameters $\boldsymbol{\theta}_n = \{\ell_n,\tau_n\}$ \cite{williams2006gaussian,pareek2021framework}. 
The VDK in \eqref{eq:VDK} incorporates the network topology by operating over the edge set $\mathcal{E}$. Consequently, for each distinct network topology or graph, a unique VDK construction $k_{\mathcal{E}}(\cdot,\cdot)$ exists, derived from the neighborhood connections in the network graph. It's important to note that both real and reactive power injections are accounted for at each node in both $\mathbf{s}$ and $\mathbf{x}_n$'s. Fig. \ref{fig:VDK}, for instance, illustrates the construction of sub-kernels for VDK for two different network configurations. Likewise, $N-2$ or $N-3$ contingencies can be modeled in the VDK design. Moreover, controlled alterations in the power network resulting from reconfiguration can also be incorporated into the VDK design.

Next section presents a transfer learning framework which exploit the VDK's ability to incorporate different network structures, appearing due to branch outages or network reconfiguration by operator.

\section{Proposed Transfer Learning Strategies}\label{sec:T1}
As discussed in the introduction, learning power flow with network contingencies is computationally hard, due to the  large number of possible network graph configurations and the combinatorial nature of contingencies. For a 30-Bus network having a total of 41 branches, there are 41 and 820 possible network structures for $N-1$ and $N-2$ contingencies, respectively. These large numbers of graph possibilities, coupled with injection uncertainties, makes it hard to accurately model the voltage-power relationships under network contingency conditions.

This section presents two strategies for efficiently learning power flow with a new network structure by leveraging existing knowledge about voltage-power relationships. We use the term `source' to refer to the dataset available with the operator and `target' to indicate training a model for a new network structure, which is not present in the source. This situation aligns with practical scenarios where operators have power flow data for only a few contingencies, often referred to as critical contingencies \cite{du2019achieving}. Our idea involves utilizing the source dataset to efficiently learn the voltage-power function $f(\mathbf{s})$ for a new network structure. First, we suggest using the hyperparameters of the source power flow models\footnote{Here, `source power flow models' refer to having comprehensive data needed to calculate the predictive mean \eqref{eq:vmean} and variance \eqref{eq:vsigma} for various network structures when dealing with load uncertainties. This data encompasses the VDK $k_{\mathcal{E}}$ and corresponding optimal hyperparameters $\boldsymbol{\theta}_{\mathcal{E}}$, acquired during the training of the VDK-GP model for each source network structure $\mathcal{E}$.} to hot start the learning process for target function hyperparameter learning. Then a subsection presents strategy to use source VDKs along with hyperparameter hot starting to learn target voltage-power function in data efficient manner.

Employing a hyperparameter hot-starting approach for GP learning, is well-founded as GP regression performance is kernel function dependent, which is parameterized using hyperparameters $\boldsymbol{\theta}$. 
A VDK, denoted as $k_{\mathcal{E}}(\cdot,\cdot)$, captures the underlying assumptions about the voltage function's behavior and its dependence on power injections for a network configuration $\mathcal{E}$. 
Further,as MLE aims to find optimal hyperparameters for model fitting, 
it is apparent that transferring a hyperparameter vector will imply having a starting point with higher likelihood during optimization. Moreover, from the structure of VDK \eqref{eq:VDK}, the total number of hyperparameters and their positioning in the kernel structure is network topology independent. This is because the set of graph branches/lines, denoted as $\mathcal{E}$, only affects kernel inputs in the form of $\mathbf{x}_n$ and not the structure of hyperparameter vector. We 
initialize the hyperparameters of the target function VDK with the average of source VDK hyperparameters as
\begin{align}\label{eq:theta}
 \boldsymbol{\theta}^o = \frac{\boldsymbol{\theta}^\star_1 + \dots + \boldsymbol{\theta}^\star_{M_s}}{M_s}
\end{align}
here, $M_s$ represents the number of source models available with different network configurations. $\boldsymbol{\theta}^\star_{m}$ denotes the optimal VDK hyperparameters for the $m$-th network configuration, while $\boldsymbol{\theta}^o$ represents the initial hyperparameters for the target VDK. This HTL strategy is designed to provide advantages in terms of faster convergence in both sample and iteration senses, relying on knowledge transfer related to power flow physics. This transfer predominantly involves the 'relative contribution' of different node injections to node voltage.

The proposed HTL has limitations in transferring power flow knowledge because the hyperparameter vector does not directly capture topological information. In the next section, we introduce Multi-Task VDK-GP, which leverages both source kernels and HTL to enhance performance when dealing with unsampled network topologies.

\subsection{Multi-Task VDK-GP (MT-VDK-GP)}\label{sec:T2}
The concept behind Multi-Task VDK-GP involves designing a \textit{combined} kernel, a strategy also explored with various kernel structures in \cite{alvarez2012kernels,tighineanu2022transfer}. In standard multi-task GPs \cite{tighineanu2022transfer}, \textit{coregionalization matrices} are employed to establish relationships within and between different data sets (source and target), where these data sets represent various contingency cases in our context. However, it's essential to note that the number of \textit{coregionalization} hyperparameters grows cubically with the number of sources used for transfer learning, leading to a significant increase in the total number of hyperparameters (cubic matrix+linear noise+kernel hyperparameters) \cite{tighineanu2022transfer}. This growth is a limitation for efficient hyperparameter learning. 

We design the MT-VDK, where the total number of hyperparameters for target learning is $2|\mathcal{N}|+2$. Here $2|\mathcal{N}|$ corresponds to VDK hyperparameters from \eqref{eq:VDK} and square exponential kernel for an $|\mathcal{N}|$-Bus network. The number of hyperparameters grows linearly with the network's size, and reduction strategies, as detailed in \cite{pareek2023graph}, can be applied to further decrease the hyperparameter vector's size. To reduce training complexity, we propose designing the MT-VDK-GP in alignment with standard system operation practices. Operators can analyze a few source contingencies in detail using a large number of samples to learn the voltage-power functions for these scenarios. However, the target function learning process must be efficient in terms of sample requirements and complexity. Hence, we introduce MT-VDK-GP, enabling distinct learning for both source and target functions. The proposed MT-VDK is given as
\begin{align}\label{eq:mtl}
 k_{\mathcal{E}_T}(\mathbf{s}^i,\mathbf{s}^j) = \sum^{|\mathcal{N}|}_{n =1} k_n(\mathbf{x}^i_n,\mathbf{x}^j_n) + \omega \sum^{M_s}_{m=1} k_{\mathcal{E}_m}(\mathbf{s}^i,\mathbf{s}^j)
\end{align}
here, neighborhood kernel $k_n(\cdot,\cdot)$ has structure dependent on target network structure $\mathcal{E}_T$ ($\{\mathbf{x}_n = \{\mathbf{s}_m|(n,m) \in \mathcal{E}\}$) while $k_{\mathcal{E}_m}$ for $m=1 \dots M_s$ are source kernels having constant structure (based on $\mathcal{E}_m$) and constant (optimal) hyperparameters $\boldsymbol{\theta}^\star_m$ based on the source learning stage. The constant source kernel means that hyperparameters associated with these kernels does not change during the target function learning process. The target hyperparameters vector consist of $\tau_n$'s, $\ell_n$'s of neighborhood kernels $k_n(\cdot,\cdot)$ in \eqref{eq:mtl}, noise variance $\sigma_n$ from A in \eqref{eq:vmean} and source kernel aggregated weight $\omega$ taking total hyperparameters to $2|\mathcal{N}|+2$.

\begin{note}
    Multi-task GP model training complexity \cite{tighineanu2022transfer} scales as $\mathcal{O}((N_t+N_s)^3)$, where $N_t$ ($N_s$) represents the number of training samples used for target (source) function learning. As a result of using \eqref{eq:mtl} in MT-VDK-GP, the training complexity for learning the target function is $\mathcal{O}(N_t^3+N_s^3)$. It's important to note that this complexity is significantly lower than that of multi-task GP learning, as the ``cube of the sum" grows much faster than the ``sum of cubes."
\end{note}

Below, in a remark, we describe a variation of the MT-VDK \eqref{eq:mtl} that offers higher expressiveness by increasing the number of hyperparameters. The next section presents an application of the proposed MT-VDK-GP for developing Probabilistic Voltage Envelopes (PVE) with varying topologies.

\begin{remark}
 In the multi-task VDK-GP scheme, source kernels can also be included with individual tuneable weight parameter $\omega_m$ i.e. $$k_{\mathcal{E}_T}(\mathbf{s}^i,\mathbf{s}^j) = \sum^{|\mathcal{N}|}_{n =1} k_n(\mathbf{x}^i_n,\mathbf{x}^j_n) + \sum^{M_s}_{m=1} \omega_mk_{\mathcal{E}_m}(\mathbf{s}^i,\mathbf{s}^j)$$ This will increase number of hyperparameters for target function learning, and will provide more flexibility for model learning. 
\end{remark}

\section{Probabilistic Voltage Envelopes (PVE)}
We consider the problem of constructing PVEs for a power network for all $N-k$ contingencies and for a given load distribution. We define PVE as the maximum (and minimum) value of a node voltage for various network topologies and a load distribution with a given probability \cite{liu2022using}.  The proposed MT-VDK-GP models can alleviate this computational burden through: i) data-efficient learning of voltage-power functions and ii) faster evaluation of voltage for a large number of power injections to compute PVEs. These envelopes can then be directly used to assess critical contingencies \cite{486117} and solve stochastic optimization efficiently \cite{xu2012optimization}.

Further, a straightforward method to construct PVEs is to solve a large number of power flows for each network topology and then develop probabilistic guarantees based on the number of power flows solved. One such statistical guarantee comes from the worst-case performance theorem \cite{alamo2010sample,tempo1996bounds}, in power system's context:
\begin{theorem}\label{thm} 
For i.i.d. samples $\mathbf{s}=\{\mathbf{s}^1 \dots \mathbf{s}^T\}$ drawn from the same net-load sub-space as per a probability distribution $\mathbb{P}_s$, let the estimated maximum voltage be defined as
\begin{align*}
 \widehat{\beta} = \max_{i= 1 \dots T} f(\mathbf{s}^i).\,\, \text{Then if}\,\,T \geq \frac{\ln{{1}/{\delta}}}{\ln{\frac{1}{1-{\epsilon}/{2}}}},
\end{align*}
 ACPF solutions (voltage samples) are used, we have
 $$ \mathbb{P}_T\big \{\mathbb{P}_s \big \{ f(\mathbf{s}) < \widehat\beta \big \} \geq 1-{\epsilon}\big/{2} \big \} \geq 1-\delta,$$
with $\delta$ being confidence and $\epsilon$ being accepted violation probability.
\end{theorem}
\begin{proof}
The proof follows directly from the worst-case performance analysis result from \cite{tempo1996bounds,alamo2010sample}. 
\end{proof}

By constructing similar result for lower voltage limits, according to Theorem \ref{thm}, $T$ i.i.d. samples can be used to obtain PVE, for a fixed network topology, with $\epsilon$ accepted violation probability and $\delta$ confidence. As discussed in the introduction, due to the large number of possible contingencies and load scenarios involved, computing these PVEs using power flow is highly computationally burdensome. With this approach, a total  of ${TN!}/{k!(N-k)!}$ power flow solutions will be required  for all $N-k$ contingency situations where $N$ is number of branches. For example, if we consider $k =2$ ($N-2$ situation), for a 30-Bus system having 41 branches \cite{pglib}, total number of power flow solutions required will be $820T$. Now $T = 1000$ is sufficient to obtain PVEs with $\delta = 10^{-4}$ confidence and $2\%$ acceptable value of violation $\varepsilon$. This implies that to obtain the PVEs for $N-2$ contingencies for 30-Bus system total $82 \times 10^{4}$ power flow samples are required, which posses significant computational burden.

We propose to use MT-VDK-GP to reduce the computational burden by training power flow models for different network topologies and then using those trained models for evaluation of power flow. Moreover, to accommodate approximation error, we use predictive variance and $\beta$ estimation in Theorem \ref{thm} is performed as
\begin{align}\label{eq:beta}
 \beta = \max_{i= 1 \dots T} \, \mu(\mathbf{s}^i) + \kappa\sigma_f(\mathbf{s}^i)
\end{align}

\begin{figure*}[t]
 \centering
 \includegraphics[width=\textwidth]{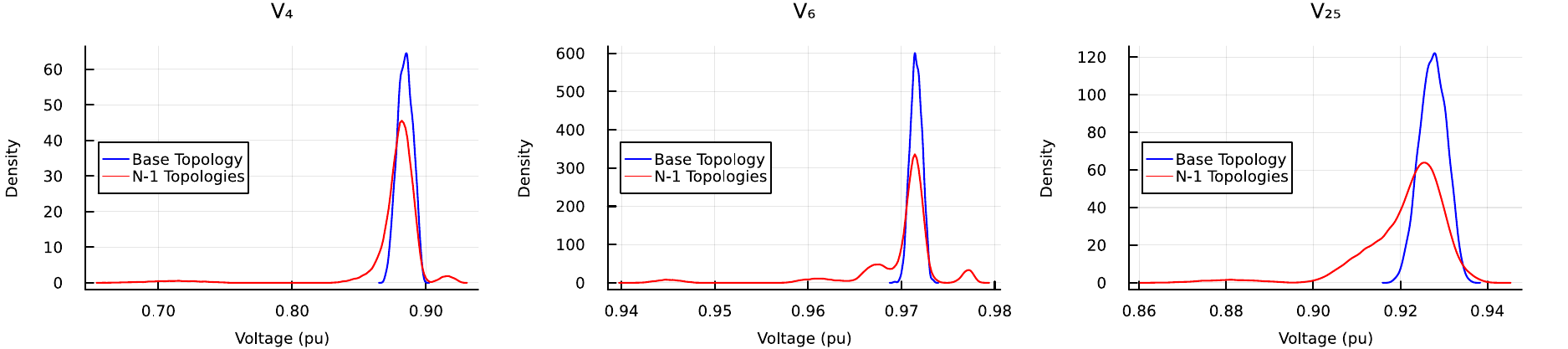}
 \vspace{-1.6em}
 \caption{Different voltage distributions across all total 38 different N-1 topologies for IEEE 30-Bus system. Total 2000 power flow solutions, at different nodal load vectors within $\pm10\%$ load hypercube, are considered at each topology as well as with base case (complete 41-line structure).}
 \label{fig:voltages}
 % \vspace{-0.5em}
\end{figure*}

Now considering that GP has been able to model the power flow correctly, $\kappa$ will determine the confidence level adjustment required in Theorem \ref{thm}. Below we present a corollary, using tail probability or quantile $\gamma(\kappa)$, to adjust it as 
\begin{corollary}\label{cor}
If estimation of maximum voltage is done using $\beta$ \eqref{eq:beta} instead of $\widehat\beta$ from Theorem \ref{thm}, then
$$ \mathbb{P}_T\Big \{\mathbb{P}_s \big \{ f(\mathbf{s}) < \beta \big \} \geq 1-\dfrac{\epsilon}{2} \Big \} \geq (1-\delta)(1-\gamma(\kappa))^T$$
\end{corollary}
\begin{proof}
The violation probability $$\mathbb{P}_s \big \{ f(\mathbf{s}) < \widehat\beta \big \} \geq 1-{\epsilon}/{2}$$ in Theorem \ref{thm} is true with respect to $\beta$ if $\widehat\beta \leq \beta$. Now, $\widehat\beta \leq \beta$ if true voltage solution $V(\mathbf{s}^i)$ is be bounded by $\mu(\mathbf{s}^i) + \kappa\sigma_f(\mathbf{s}^i)$ for every injection sample $\mathbf{s}^i$ when $i=1\dots T$. 
Further, assuming that true power flow function follows Gaussian distribution, probability of $V(\mathbf{s}^i) \leq \mu(\mathbf{s}^i) + \kappa\sigma_f(\mathbf{s}^i) $ is $1-\gamma(\kappa)$ for an injection vector $\mathbf{s}^i$. Here, $\gamma(\kappa)$ is tail probability from normal distribution  With i.i.d injection samples, the probability that the above inequality holds for samples all $i=1\dots T$ can be given as $(1-\gamma(\kappa))^T$. Therefore, confidence in $$\mathbb{P}_s \big \{ f(\mathbf{s}) < \widehat\beta \big \} \geq 1-{\epsilon}/{2}$$ using GP-based $\beta$ is $(1-\delta)(1-\gamma(\kappa))^T$.
\end{proof}

Now, using \eqref{eq:beta} and Theorem \ref{thm} (as Corollary \ref{cor}), the number of power flow solutions required will will be ${N_sN!}/{k!(N-k)!}$ where $N_s << T$ is the number of power flow samples required to train the models. Evaluation of these models $T$ times requires negligible additional time since they involve direct function evaluation. Further, as GP model evaluation does not depend on distribution information, a single model can be used to evaluate PVEs for different injection distributions\cite{pareek2021framework}. This will provide additional computational benefits as sampling methods (like Monte-Carlo Simulations) will require again to solve at least $T$ power flows for each possible injection distribution. See \cite{pareek2021framework,pareek2020gaussian}
for more on advantages of non-parametric nature of GP-based power flow and optimal power flow models. Note that this sample complexity benefit comes at the cost of confidence reduction by a factor of $(1-\gamma(\kappa))^T$. However, the Gaussian distributions have short tail and $\gamma(\kappa)$ decrease very fast with increase in $\kappa$.  As an example for $\kappa \geq 3.5 $ and $T = 1000 ...$ we get $(1-\gamma(\kappa))^T = 0.9157 $.

\begin{figure*}[t]
 \centering
 \includegraphics[width=\textwidth]{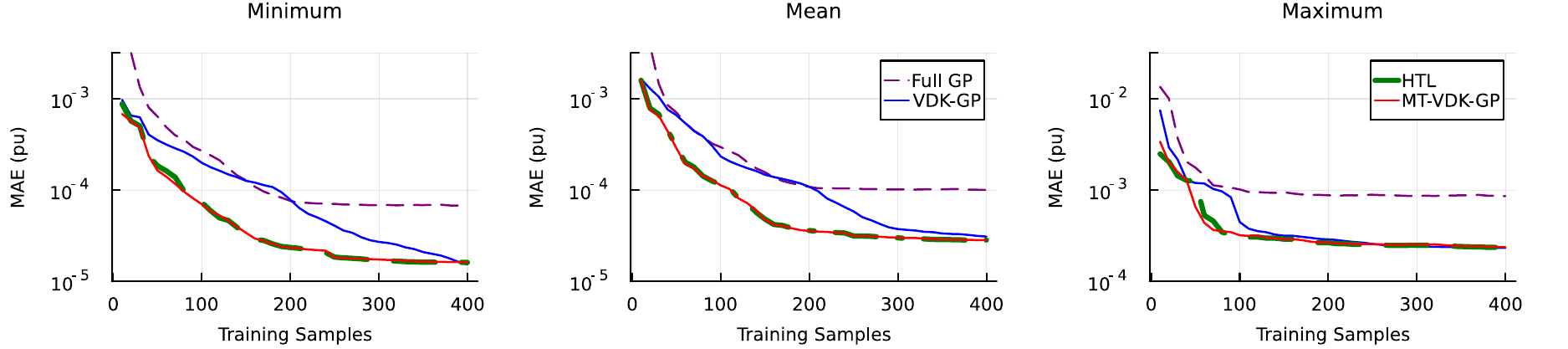}
 \vspace{-1em}
 \caption{Minimum, mean and maximum error variations with number of training samples for $V_4$ in 30-Bus system. Here, Minimum, mean and maximum is calculated across 38 different network topologies obtained after a single line outage at a time (N-1 contingency). Number of training iterations are same as training samples and 500 data points are used for testing with each network topology in out-of-sample manner. All y-axis scales are in $\log$ and note that right-most sub-figure (maximum) as different y-axis range.}
 \label{fig:m3}
 \vspace{-0.8em}
\end{figure*}

\section{Results and Discussion}
In this section we use IEEE 30-Bus system \cite{pglib} to perform all the training and testings. The 30-Bus system used has total 41 lines and out of which only 38 N-1 contingency scenarios are base case feasible. We restrict this section to analyze these 38 different topologies for N-1 contingencies while 356 different network topologies (out of total possible 820) are used for N-2 contingency training-testing. For load uncertainty, a $\pm10\%$ hypercube is considered. This implies that real and reactive power injection at each node lies within $\pm10\%$ of its base-case value \cite{pareek2021framework,pareek2023graph}. All testing is done using out-of-sample analysis for power injection inputs and number of training samples, testing samples and iterations for MLE optimization are given with respective results. All models are implemented in \texttt{Julia} and ACPF is solved using \texttt{PowerModels.jl} for sample generation. We use mean absolute error (MAE) to evaluate performance as. 
\begin{align}\label{eq:errors}
 \text{MAE} & = \frac{\sum_{i=1}^{M} \Big | \mu(\mathbf{s}^i) - \widehat{V}(\mathbf{s}^i) \Big |}{M}
\end{align}
where, $\mu(\mathbf{s}^i)$ is mean voltage prediction of that particular node voltage \eqref{eq:vmean}, $\widehat{V}(\mathbf{s}^i)$ represents the true ACPF solution values and $M$ is number of testing samples. Fig. \ref{fig:voltages} shows voltage sample densities for 38 different N-1 topologies and base network structure when 2000 input injections are used within $\pm 10\%$ (uniformly distributed). It is evident that contingencies lead to varying voltage spreads among different node voltages relative to the base network case. In next two subsections, we use these voltages to benchmark proposed methods against baseline approaches. For baseline performance, we use two cold starting GP approaches, i) \textbf{Full GP} to represent standard GP learning using square exponential kernel with both hyperparameters as $\exp\{1\}$ and $\sigma_n = \exp\{-5\}$ \cite{pareek2021framework}, ii) \textbf{VDK-GP} for GP using VDK without any transfer learning and initial hyperparameters for all sub-kernels is same i.e. $\tau_n = \ell_n = \exp\{1\}$ and $\sigma_n = \exp\{-5\}$. The proposed methods are referred as \textbf{HTL} and \textbf{MT-VDK-GP} as discussed in Section \ref{sec:T1} and \ref{sec:T2} respectively. Source sets contains the base topology power flow dataset and the same for a set of chosen contingency topologies as mentioned in the respective figure captions. 

\begin{figure}[t]
 \centering
 \includegraphics[width=0.9\columnwidth]{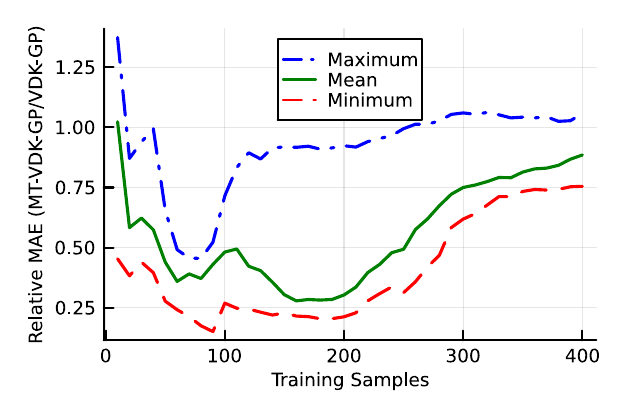}
 \vspace{-1.0em}
 \caption{Relative MAE variation with respect to training samples-- defined as the ratio of MAE obtained using MT-VDK-GP to VDK-GP. A value of 1.0 on the y-axis indicates similar performance by both methods. The training and testing conditions are the same as in Fig. \ref{fig:voltages}. Valleys in the error curves reflect the superior performance of the proposed MT-VDK-GP at low data regimes.}
 \label{fig:ratio}
 \vspace{-0.5em}
\end{figure}

\begin{figure*}[t]
 \centering
 \includegraphics[width=\textwidth]{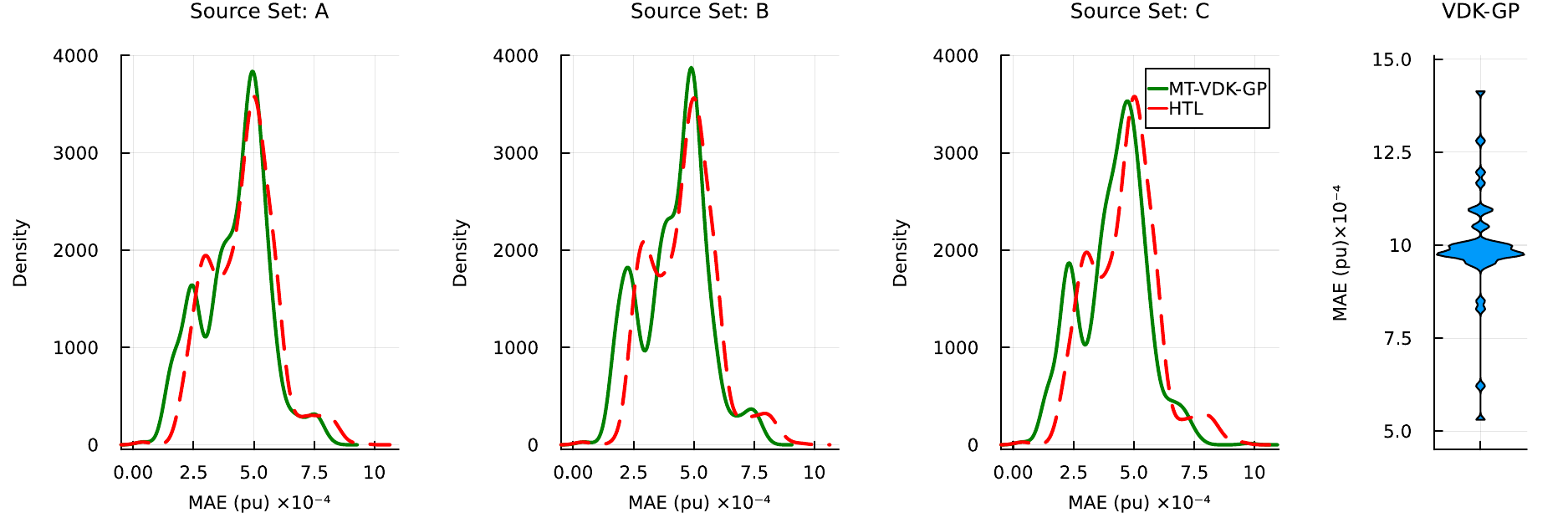}
 \vspace{-1em}
 \caption{Mean absolute error (MAE) densities for 24 node voltages for 38 different network topologies (N-1 contingencies). Different source sets are used for transfer learning. Indices of branches removed for source set $\rm A$ are 1,5 and 10; for $\rm B$ are 1,5,8 and 10 while in source set $\rm C$ branch 12, 15, 18, 22, 35 are removed to construct source dataset as discussed in Section \ref{sec:T2}. For each network topology and each node, 60 training samples are used with 50 training iterations while 1000 out-of-sample data points are used for testing. Right most sub-figure shows violin curve of MAE values when training is done using VDK-GP without any transfer learning for all nodes with all network topologies.}
 \label{fig:allnodeN1}
 \vspace{-0.5em}
\end{figure*}

\subsection{N-1 Benchmarking}
To gauge the performance of the proposed transfer learning (TL) methods comprehensively, we start this sub-section by describing benchmarking experiments conducted on $V_4$ learning, as it has the maximum variation among all node voltages in the 30-Bus system, across different N-1 network topologies.
Fig. \ref{fig:m3} shows the comparative performance in terms of MAE (pu) for four different GP learning methods. Starting from the left, the sub-figures display the minimum, mean, and maximum MAE performance across 38 distinct contingencies. To clarify, the maximum curves represent the least favorable performance observed among the 38 N-1 contingency power flow models.The first observation is that Full GP is the worst-performing approach, which is in line with findings presented in \cite{pareek2023graph}. Moreover, the rate of decrease in MAE with sample increase is higher using both TL methods compared to the VDK-GP baseline performance. This results in a faster convergence of the error curve for both HTL and MT-VDK-GP compared to VDK-GP.
Hence, Fig. \ref{fig:m3} aids in drawing the conclusion that source models do contain information on power flow physics, which can be beneficial in decreasing the MAE at a higher rate than training fresh and individual power flow models for each new network topology that has not been previously encountered.

Fig. \ref{fig:ratio} shows the ratio of MAE achieved using MT-VDK-GP with VDK-GP. It depicts a two-fold advantage in mean performance by the proposed approach over VDK-GP when the number of training samples is relatively low, lying between 50-250. The valleys establish the data efficiency of the proposed MT-VDK-GP. It can be observed in Fig. \ref{fig:m3} that the blue curve representing VDK-GP eventually coincides with that of the proposed methods. We conclude that with enough data all methods will eventually achieve the same MAE, whereas for the low sample regime (5-250 for the chosen example)  MT-VDK-GP offers superior performance.

\begin{figure*}[t]
 \centering
 \includegraphics[width=\textwidth]{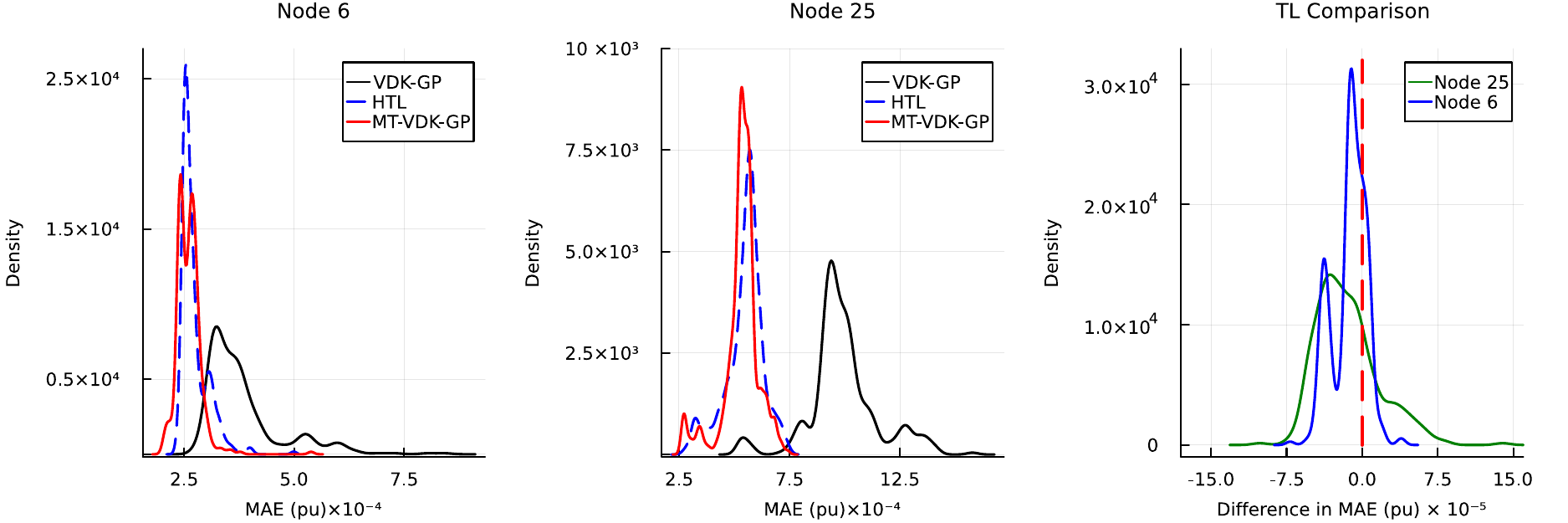}
 \vspace{-1.5em}
 \caption{Densities of MAE (first two sub-figures from left) using proposed transfer learning mechanisms and VDK-GP for 356 different network topologies arising due to simultaneous two branch outages (N-2 contingency). For each voltage, $V_6$ and $V_{25}$, 60 training sample per network topology are used along with 50 training iterations for all three different GP settings and source set for TL has branch 1 and 10 outage data along with base topology. Rightmost sub-figure depicts densities of MAE difference (MAE with MT-VDK-GP - MAE with HTL) for different node voltage functions. Higher density on the left of zero line (more than 75\% in both cases) shows superior MT-VDK-GP performance over HTL.}
 \label{fig:N2}
 \vspace{-1em}
\end{figure*}

 \begin{figure*}[t]
 \centering
 \includegraphics[width=\textwidth]{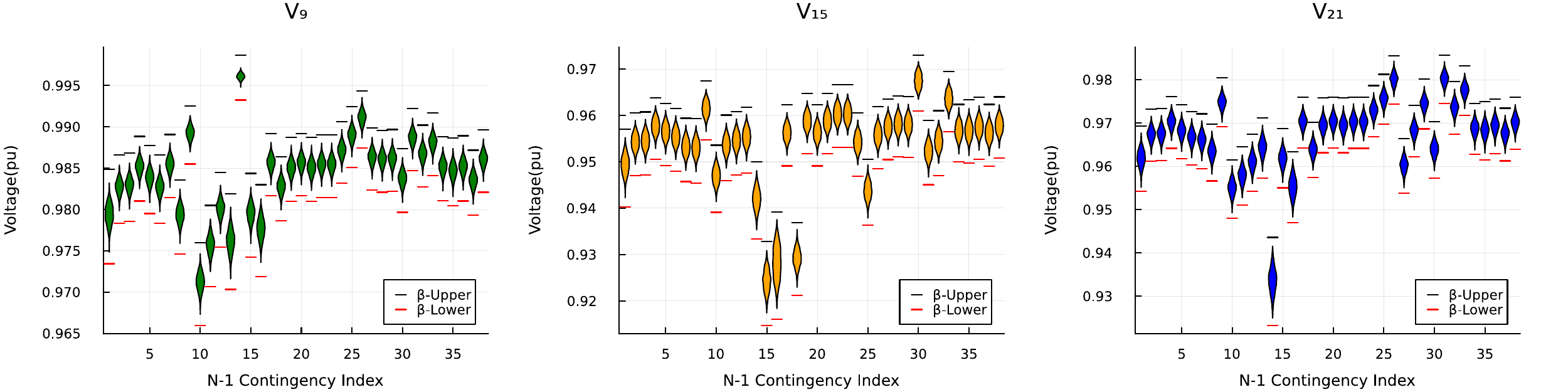}
 \vspace{-1.3em}
 \caption{Probabilistic Voltage Envelopes (PVEs) for three node voltages and 38 different network topologies arising due to N-1 contingencies in 30-Bus network. Violins represent node voltage distribution at each topology for 2000 power injections samples, out of which 60 are used for training, 1000 for $\beta$-Upper and $\beta$-lower calculation while rest are unseen injection samples. Training conditions (60 samples, 50 iterations) are same as in Fig. \ref{fig:allnodeN1} and MT-VDK-GP models are used with source set B.}
 \label{fig:pve}
 \vspace{-1em}
 \end{figure*}

\begin{table}[t]
 \centering
\caption{Area Under Density Curves for Different MAE Cut-Offs}
 \begin{tabular}{c|c|c|c|c|c|c|c|}
 & & \multicolumn{6}{c|}{Source Set} \\
 \cline{3-8}
 MAE (pu) & VDK-GP & \multicolumn{2}{c|}{A} & \multicolumn{2}{c|}{B} & \multicolumn{2}{c|}{C} \\ 
 \cline{3-8}
 $\times 10^{-4}$ & & MT & HTL & MT & HTL & MT & HTL \\
 \hline
 $< 10.0$ & 0.68 & 1.00 & 1.00 & 1.00 & 1.00 & 1.00 & 1.00\\
 $< 5.00$ & 0.00 & 0.69 & 0.60 & 0.71 & 0.60 & 0.74 & 0.59 \\
 $< 2.50$ & 0.00 & 0.15 & 0.05 & 0.17 & 0.04 & 0.18 & 0.05 \\
 \hline
 \\
 \multicolumn{8}{l}{MT := MT-VDK-GP \& Train-Test Settings same as in Fig. \ref{fig:allnodeN1}.}
 \end{tabular}
 \label{tab:area}
 % \vspace{-0.5em}
\end{table}

With different topologies, some node voltages can change significantly. Thus, it is important to measure the performance of TL methods on all node voltages. Fig. \ref{fig:allnodeN1} shows density plots of MAEs obtained for all node voltages (24 different non-generator nodes) for 38 different topologies (N-1 contingencies) in the 30-Bus system. Also, from left, the first three sub-figures are obtained when different topologies are part of the source set, i.e., available for pre-training, while the right-most figure shows the MAE distribution when VDK-GP is trained for each voltage and each topology from the baseline (no transfer learning). We observe that the proposed methods have effectively halved the mean of MAE density, shifting it from $\sim 10\times10^{-4}$ to $\sim 5\times10^{-4}$ when compared to the VDK-GP baseline. Also, the tail of MAE densities is shorter with the proposed method and does not extend much beyond $10\times10^{-4}$, as is the case with VDK-GP. The left-shifting of MT-VDK-GP-associated MAE density (green) compared to HTL MAE density (red dashed) indicates that due to source kernel addition in MT-VDK, the power flow knowledge retention is better than using only hyperparameters in HTL. Table \ref{tab:area} highlights the comparative performance shown in Fig. \ref{fig:allnodeN1} by calculating the area under different density curves for different MAE cut-offs. It validates the observations made above regarding MAE density being below $10\times10^{-4}$ and the majority of voltage-topology instances having error less than $5\times10^{-4}$ with proposed methods. 

An important observation from Table \ref{tab:area} is that HTL method performance is nearly source-set-agnostic. This is because hyperparameters do not contain network topology information directly. Therefore, the proposed MT-VDK-GP has better performance compared to HTL, and performance improvement can be observed with an increase in source-set size going from set A to C in Table \ref{tab:area}. Note that source set elements, along with size, will also impact the performance of MT-VDK-GP. The selection of a source set can be done based on critical contingency (or branches connected to the largest source/load), if the operator has the option of collecting data by choice. An intelligent source set selection method and its impact on the performance of MT-VDK-GP require a separate work and will be part of the future studies in this direction.

\subsection{N-2 Benchmarking}
We analyze the performance of proposed methods for network topologies arising due to simultaneous double branch outage (N-2 contingencies). In Fig. \ref{fig:N2}, first two sub-figures from left shows MAE density curves for total 356 different network structures arising due to double line outage (N-2) contingencies. It is clear that proposed TL methods achieve lesser error for both voltages. Here note that source set does not contain any data on N-2 topologies. Therefore, the transfer capability is shown in Fig. \ref{fig:N2} is for completely out-of-sample network structures. Also, similar to the observations made with Fig. \ref{fig:allnodeN1}, the proposed methods also reduces tail of MAE distributions which implies lower worst case error. The right most sub-figure in Fig. \ref{fig:N2} compares the performance of proposed TL methods with each other by calculating the difference in MAE. As for majority of network topologies the MT-VDK-GP outperforms the HTL, density of difference MAE\footnote{Difference MAE = $\text{MAE}_{\texttt{MT}} - \text{MAE}_{\texttt{HTL}}$} is higher on left hand side of zero line. Among 356 topology instances, MT-VDK-GP has lesser error compared to HTL in 276 (77.52\%) and 268 (75.28\%) instances for node 6 and node 25 voltage respectively. These results shows that providing higher and direct topological information leads to better learning performance for N-2 contingencies, even when only N-1 contingency data is available for source training. 

\subsection{Probabilistic Voltage Envelopes (PVEs)}
The goal of PVEs is to provide limits within which the voltages will remain, with a given probability and confidence. Fig. \ref{fig:pve} shows PVEs for 38 different network topologies, constructed using MT-VDK-GP. We considered $\kappa = 3.75$ in Corollary \ref{cor}, which provides $\gamma=8.8\times10^{-5}$ and $\delta$ is taken as $10^{-4}$. Thus, from Corollary 1 using 1000 samples for $\beta$ estimation, we obtain a violation probability of less than 1.9\% with more than 90\% confidence. Moreover, from Fig. \ref{fig:pve}, it is clear that the proposed method has been able to get upper ($\beta$-Upper) and lower ($\beta$-Lower) estimates of the maximum voltage, which bounds all 2000 true voltage solutions, in all network structure cases. Also, it takes only $3.3\pm 0.6$ sec. to develop PVEs for one node voltage on a given contingency along with the time to generate 60 samples. The total number of power flow solutions required will be $60\times38$ using the proposed method, which is more than sixteen times less than $1000\times38$ samples required using brute force MCS. Further, by increasing training samples for target functions in MT-VDK-GP beyond 60, we can tighten the bounds shown in Fig. \ref{fig:pve}. Moreover, with a change in injection distribution, the MT-VDK-GP model can be used without re-training (only evaluation) while the brute force MCS method will require resolving 1000 power flows again. Table \ref{tab:data} shows the data-efficiency of proposed MT-VDK-GP method, compared to both MCS and VDK-GP methods.

\begin{table}[t]
    \centering
    \caption{Data-Efficiency of Proposed Method in Terms of Power Flow Sample Requirement}
    \begin{tabular}{c|ccc}
 Method  & MCS & VDK-GP & MT-VDK-GP  \\
\hline
   N-1 &  $38000$ & $3900$ &$2380$ \\
   N-2 & $356000$ & $35700$ & $21460$ \\
         \hline
    \end{tabular}
    \label{tab:data}
\end{table}

\section{Conclusions}
In summary, this work presents data-efficient power flow learning techniques designed for network contingencies. Leveraging the vertex-degree kernel (VDK) for Gaussian process (GP) learning, we harness existing power flow data and models to effectively learn voltage-load functions in previously unencountered network configurations. Our proposed multi-task vertex degree kernel (MT-VDK) combines historical VDKs with parameterized VDKs for unseen networks, offering a computationally efficient alternative to multi-task GP (MTGP) methods. Additionally, we introduce a simple hyperparameter transfer learning (HTL) approach for employing historically-trained power flow models when encountering previously unencountered network contingencies. Using MT-VDK-GP models for voltage-power functions, we have developed a sample-efficient mechanism for probabilistic voltage envelopes (PVEs). Simulation results on the IEEE 30-Bus network demonstrate the ability of historically-trained models to retain and transfer power flow insights in N-1 and N-2 contingency scenarios. Future directions for this research could include extending the application of MT-VDK-GP to larger power networks and exploring its potential in power dispatch adjustment after contingency realization under load uncertainty.

\bibliographystyle{IEEEtran}
\bibliography{main}

\end{document}